\documentclass{article}

\PassOptionsToPackage{numbers, compress}{natbib}

    \usepackage[preprint]{neurips_2020}

\usepackage[utf8]{inputenc} 
\usepackage[T1]{fontenc}    
\usepackage{hyperref}       
\usepackage{url}            
\usepackage{booktabs}       
\usepackage{amsfonts}       
\usepackage{nicefrac}       
\usepackage{microtype}      
\usepackage{graphicx}
\usepackage{subcaption}
\usepackage{amsmath}
\usepackage{amsthm}
\usepackage{enumitem}
\usepackage{hyperref}
\usepackage{multirow}
\usepackage{wrapfig}
\newtheorem{theorem}{Theorem}

\newtheorem{lemma}{Lemma}
\newtheorem{proposition}{Proposition}

\RequirePackage{algorithm}
\RequirePackage{algorithmic}
\newcommand\sbullet[1][.5]{\mathbin{\vcenter{\hbox{\scalebox{#1}{$\bullet$}}}}}
\newenvironment{hproof}{%
  \proof}{\endproof}

\title{
Improved Bilevel Model: Fast and Optimal Algorithm with Theoretical Guarantee
}

\author{
 Junyi Li \\
  University of Pittsburgh\\
  \texttt{junyili.ai@gmail.com}\\
  \And
  Bin Gu\\
  Nanjing University of Information Science \& Technology\\
  \texttt{jsgubin@gmail.com}\\
  \AND
  Heng Huang \\
  University of Pittsburgh \& JD Finance America Corporation\\
  \texttt{ henghuanghh@gmail.com}\\
}

\begin{document}

\maketitle

\begin{abstract}
Due to the hierarchical structure of many machine learning problems, bilevel programming is becoming more and more important recently, however, the complicated correlation between the inner and outer problem makes it extremely challenging to solve. Although several intuitive algorithms based on the automatic differentiation have been proposed and obtained success in some applications, not much attention has been paid to finding the optimal formulation of the bilevel model. Whether there exists a better formulation is still an open problem. In this paper, we propose an improved bilevel model which converges faster and better compared to the current formulation. We provide theoretical guarantee and evaluation results over two tasks: Data Hyper-Cleaning and Hyper Representation Learning. The empirical results show that our model outperforms the current bilevel model with a great margin. \emph{This is a concurrent work with \citet{liu2020generic} and we submitted to ICML 2020. Now we put it on the arxiv for record.}
\end{abstract}

\section{Introduction}
Bilevel programming \cite{willoughby1979solutions} defines a type of optimization where one problem is nested into another problem. Many machine learning problems can be united under the framework of bilevel programming, for example, hyper-parameter optimization~\cite{pedregosa2016hyperparameter, bergstra2011algorithms, lacoste2014sequential, luketina2016scalable, baydin2017online}, few-shot learning~\cite{finn2017model, koch2015siamese, vinyals2016matching, santoro2016meta, mishra2017simple} and the generative adversarial network~\cite{goodfellow2014generative, arjovsky2017wasserstein, gulrajani2017improved, brock2018large, radford2015unsupervised}. As a result, the bilevel programming and its applications in machine learning~\cite{maclaurin2015gradient, luketina2016scalable, franceschi2017forward, baydin2017online, lorraine2018stochastic} have drawn more and more attention recently, however, bilevel programming is much more challenging than classical optimization problems due to the complicated correlation between the outer problem and the inner problem, whose difficulty is compounded in the high-dimensional domain. Since the explicit optimal solution is usually not available, some intuitive approximate gradient-based methods have been proposed, such as reversible learning \cite{maclaurin2015gradient}, the forward-HG and reverse-HG \cite{franceschi2017forward}, truncated back-propagation \cite{shaban2018truncated}. 
These methods acquire impressive results when applied in the hyper-parameter optimization \cite{pedregosa2016hyperparameter, franceschi2017forward, shaban2018truncated} and few shot learning \cite{franceschi2018bilevel}. Despite the advance in algorithms and applications, not much attention has been paid to the formulation of the bilevel model, and whether there exists a better formulation that can lead to better results and faster convergence is still an open problem. 

In this paper, we focus on improving the current formulation of the bilevel model. Note in the process of solving a bilevel programming problem (the exact form will be introduced shortly), the first step is to solve the inner problem, however, the optimal way to solve it is still an open problem. In the current formulation, it simply states to choose the minimizer of the inner problem. In fact, this is both sub-optimal and unrealistic. On the one hand, the outer problem is the optimization objective, therefore, optimizing the inner problem solely for its own sake is sub-optimal from the outer problem's perspective, on the other hand, it is impossible to get the "optimum" if the inner problem has multiple equally good local minimizers. So we should come up with a better way to solve the inner problem, and a natural choice is to regularize the inner problem with the outer problem, more specifically, we want the inner problem to converge to a point that is not only optimal for the inner problem itself but also is good for the outer problem. Intuitively, this can filter out the "bad" values of the inner variable and therefore accelerate the convergence of the outer problem. We will elaborate more of the idea in the subsequent sections.
The main contributions of this paper are summarized as follows:
\begin{enumerate}[label=(\roman*)]
    \item{We propose an improved bilevel model with better and faster convergence and an efficient algorithm to solve this new model;}
    \item{We prove the convergence of the proposed algorithm and also provide a lower bound result showing the superior performance of our model;}
    \item{We conduct extensive experiments and ablation studies over different tasks and data-sets to demonstrate the effectiveness of our improved bilevel model.}
\end{enumerate}

\paragraph{Organization.}The remaining part of this paper is organized as follows: in Section~\ref{sec:rw}, we introduce the related work including the current bilevel model and commonly used algorithms solving it; in Section~\ref{sec:ibm}, we formally define our improved bilevel model and also propose an efficient algorithm to solve this model; in Section~\ref{sec:ta}, we present the theoretical analysis of the proposed model and the algorithm; in Section~\ref{sec:er}, we show experiments and ablation studies and make discussions; in Section~\ref{sec:c}, we summarize the paper and discuss future directions.

\paragraph{Notations.}In this paper, we denote model parameters with Greek lowercase letters (such as $\lambda$), vector space with Greek uppercase letters (such as $\Lambda$) and real valued functions with Latin lowercase letters (such as f and g). Moreover, we use $\nabla$ to denote the gradient of a function, while for partial derivatives, we either use $\nabla_k$ to represent the partial derivative with respect to the $k_{th}$ variable, or $\nabla_{\lambda}$ to represent the partial derivative with respect to the variable $\lambda$. Higher order derivatives follow similar rules. Finally, we use $[N]$ to denote the sequence of integers from 1 to N, and $||\sbullet||$ the euclidean norm for vectors or $\ell_2$-norm for matrices.

\section{Related Work}
\label{sec:rw}
In the mathematics and optimization literature, bilevel programming dates back to at least the 1960s when \citet{willoughby1979solutions} proposes a regularization method to solve the following problem:\\
\begin{equation}
\label{eq:opt-bilevel}
\begin{split}
    \underset{\omega\in \Omega^{*}}{\min}\ g(\omega)\;\;\;  s.t.\;\;\Omega^{*} = \{\omega| \omega\in \underset{\omega\in \Omega}{\arg\min}\ h(\omega)\} \\
\end{split}
\end{equation}
This problem is called "bilevel" as the constraint itself is an optimization problem. It has been extensively investigated since then \cite{ yamada2011minimizing, ferris1991finite, solodov2007explicit, sabach2017first, xu2004viscosity}. Particularly, \citet{sabach2017first} proposes a first order method called BiG-SAM to solve Eq.~(\ref{eq:opt-bilevel}), which allows the bilevel programming to be solved as like a "single level" optimization problem. 
The algorithmic details of the BiG-SAM  and other methods can be found in Appendix A. 

Recently, there is a rise of interest in the bilevel programming in the machine learning field due to its potential application in many problems.
In machine learning, we study a more complicated formulation involved with two sets of variables~\cite{couellan2016convergence, ghadimi2018approximation}:
\begin{equation}
\label{eq:classical-bilevel}
    \underset{\lambda \in \Lambda, \omega\in \Omega}{\min}\ g(\omega,\lambda)
      \;\;\;  s.t.\;\;  \omega = \underset{\omega\in \Omega}{\arg\min}\ h(\omega,\lambda)
\end{equation}
For the ease of reference, we will refer to $\omega$ as the inner variable and $h(\omega,\lambda)$ as the inner problem, similarly refer to $\lambda$ as the outer variable and $g(\omega,\lambda)$ as the outer problem. What's more, we will refer to Eq.~(\ref{eq:classical-bilevel}) as the basic bilevel model in contrast to our model in the subsequent discussions. The best-known methods for solving the basic bilevel model are gradient-based methods \cite{franceschi2017forward, shaban2018truncated, franceschi2018bilevel}, which gain popularity because of its simplicity and scalability. These methods typically apply back-propagation through time and automatic differentiation to compute the gradients. Based on how the gradients are stored and calculated, they can be divided into the forward-mode methods or the reverse-mode methods \cite{franceschi2017forward}. Roughly speaking, the reverse-mode runs faster, while the forward-mode consumes less storage. Based on if the back-propagation path are fully included, they can be divided into truncated or full methods \cite{shaban2018truncated}. More introduction of these methods can be found in Appendix A.

\section{Improved Bilevel Model}
\label{sec:ibm}
\subsection{New Bilevel Programming Formulation}

Most previous research work focuses on proposing a better algorithm to solve the basic bilevel model Eq.~(\ref{eq:classical-bilevel}), however, we instead focus on finding a new formulation of bilevel model which has better and faster convergence than Eq.~(\ref{eq:classical-bilevel}), though we do introduce an efficient algorithm for our improved formulation in Section~\ref{sec:alg}. To the best of our knowledge, we are the first to explicitly address this improved formulation, instead of the original one.
More precisely, we propose to solve the following model:
\begin{equation}
\label{eq:hier-bilevel}
\begin{split}
      \underset{\lambda \in \Lambda, \omega\in \Omega}{\min}\  g(\omega,\lambda)
      \;\;\;  s.t.\;\;   \omega = \underset{\omega\in \Omega^{*}_{\lambda}}{\arg\min}\ g(\omega,\lambda) \;\;\; where \;\;\; \Omega^{*}_{\lambda} = \{\omega, \omega\in \underset{\omega\in \Omega}{\arg\min}\ h(\omega,\lambda)\}
\end{split}
\end{equation}
Compared to the formulation in Eq.~(\ref{eq:classical-bilevel}), the outer problem $g(\omega,\lambda)$ shows twice in our model. One is in the objective which is the same as in Eq.~(\ref{eq:classical-bilevel}), while the other is shown in the constraints, which differs from the original formulation and is the key innovation of our model. We select $\omega$ in the set $\Omega^{*}_{\lambda}$ that achieves the lowest outer objective value, in other words, we encourage the model to explore the "good" $\omega$ and avoid the "bad" ones. 

The intuition behind this design is as follows: With the regularization, the outer problem narrows down its search space to where the outer objective value is small. Naturally, this smaller and better search space will accelerate the convergence. From another perspective, this leads to more efficient use of the outer problem's information. In the original basic formulation, the signal of the outer problem is back-propagated only if the inner problem is (approximately) solved. In contrast, the outer problem passes information frequently during the inner problem solving process in our model. Not surprisingly, our model will converge faster and better compared to the original one. The theoretical analysis and empirical results will be presented in the subsequent sections. 

\subsection{Proposed Algorithms}
\label{sec:alg}
In this section, we present an efficient gradient-based algorithm to solve Eq.~(\ref{eq:hier-bilevel}). As shown in Algorithm~\ref{alg:overall-dynamic}, we apply the gradient descent step (Line 10) to update the outer variable $\lambda$ within each iteration, where G is an estimation of the hyper-gradient with respect to $\lambda$. There are two phases to calculate $G$. Firstly, we solve the inner problem (Line 4), then we apply automatic differentiation to compute the gradient (Line 5 - 9). In the first phase, we solve the following inner problem:
\begin{equation}
\label{eq:inner}
\begin{split}
 \underset{\omega\in \Omega^{*}_{\lambda}}{\arg\min}\ g(\omega,\lambda) \;\;\; where \;\;\; \Omega^{*}_{\lambda} = \{\omega, \omega\in \underset{\omega\in \Omega}{\arg\min}\ h(\omega,\lambda)\}
\end{split}
\end{equation}.
In fact, Eq.~(\ref{eq:inner}) degenerates to Eq.~(\ref{eq:opt-bilevel}) when $\lambda$ is fixed, so it can be solved efficiently with a first order method such as the BiG-SAM, we adopt BiG-SAM \cite{sabach2017first} as it is both simple to implement and achieves SOTA convergence rate. The procedure to solve the inner problem is in Algorithm~\ref{alg:inner_dynamic}. Note that we evaluate the gradient with respect to the outer function (Line 6 of Algorithm~\ref{alg:inner_dynamic}) within each BiG-SAM step to regularize the inner problem, what's more, an interesting point deserves to mention is that Algorithm~\ref{alg:inner_dynamic} degenerates to a gradient descent optimizer when we set $\alpha_k$ as 1 and we can use it to solve the inner problem of the basic bilevel model. 
We will perform an ablation study later to study this phenomenon.

 \begin{minipage}{0.5\textwidth}
 \vspace{-0.8cm}
 \centering
\begin{algorithm}[H]
  \caption{Optimization of inner problem}
  \label{alg:inner_dynamic}
\begin{algorithmic}[1]
  \STATE {\bfseries Input:} Initial inner variable $\boldsymbol{\omega_{0}}$;\ current outer variable $\boldsymbol{\lambda}$;\ number of iterations \textbf{K};\ hyper-parameters \textbf{t},\ \textbf{s}\ and $\boldsymbol{\alpha}$;
  \STATE {\bfseries Output:} Optimal inner variable $\boldsymbol{\hat{\omega}_{\lambda}}$;
  \FOR{$k = 0 \to K$} \STATE (BiG-SAM step)
    \STATE $\theta_{k+1} = \omega_{k} - t\nabla_{1} h(\omega_{k},\lambda)$
    \STATE $\phi_{k+1} = \omega_{k} - s\nabla_{1} g(\omega_{k},\lambda)$
    \STATE $\omega_{k+1} = \alpha_{k+1}\theta_{k+1} + (1 - \alpha_{k+1}) \phi_{k+1}$
  \ENDFOR
  \STATE {\bfseries Return} $\hat{\omega}_{\lambda} = \omega_{K}$
\end{algorithmic}
\end{algorithm}
 \end{minipage}
 \begin{minipage}{0.5\textwidth}
 \vspace{-0.3cm}
 \centering
\begin{algorithm}[H]
  \caption{Optimization of outer problem (main algorithm)}
  \label{alg:overall-dynamic}
\begin{algorithmic}[1]
  \STATE {\bfseries Input:} Initial outer variable $\boldsymbol{\lambda_0}$, learning rate $\boldsymbol{\eta}$
  \STATE {\bfseries Output:} Optimal outer variable $\boldsymbol{\lambda_{T}}$;
  \FOR{t = 1 to T\ -\ 1}
  \STATE Initialize inner variable $\boldsymbol{\omega_0}$ and invoke \textbf{Algorithm~\ref{alg:inner_dynamic}} to get $\hat{\omega}_{\lambda}$ (solve inner problem);
  \STATE Let $\alpha = \nabla_{1} g(\hat{\omega}_{\lambda}, \lambda_t)$, $G = \nabla_{2} g(\hat{\omega}_{\lambda}, \lambda_t)$
  \FOR{k = K - 1 $\to$ 1}
  \STATE $G = G + \alpha \nabla_2 \Phi(\omega_{k+1},\lambda_t)$
  \STATE $\alpha = \alpha \nabla_1 \Phi(\omega_{k+1},\lambda_t)$
  \ENDFOR
  \STATE $\lambda_{t+1} = \lambda_{t} - \eta G$ 
  \ENDFOR
  \STATE {\bfseries Return} $\lambda_T$
\end{algorithmic}
\end{algorithm}
\end{minipage}

In the next phase, we evaluate hyper-gradient with automatic differentiation. Suppose we solve the inner problem with $K$ steps, then its dynamics is formally written as:
\begin{equation}
\label{eq:re-bilevel}
\begin{split}
    \hat{\omega}_{\lambda} = \omega_K,\; \omega_{k+1} = \Phi(\omega_{k},\lambda),\ k \in [K]
\end{split}
\end{equation}
where $\Phi(\omega_{k},\lambda)$ denotes the dynamics of BiG-SAM step.
A procedure to calculate the hyper-gradient based on the reverse mode is shown in Line 5 - 9 of Algorithm~\ref{alg:overall-dynamic} (the derivation of this can be found in Appendix B.1). The convergence property of Algorithm~\ref{alg:overall-dynamic} will be shown in Theorem~\ref{th:convergence} of Section~\ref{sec:ta}.

 
\paragraph{Complexity Analysis.}Now we analyze the time complexity of our algorithm. For simplicity, we analyze the time complexity of calculating hyper-gradient in terms of the inner iterations K (the number of iterations in Algorithm~\ref{alg:inner_dynamic}). In our algorithm, calculation of hyper-gradient is divided into two phases. In phase 1, we calculate $\hat{\omega}_{\lambda}$ in Algorithm~\ref{alg:inner_dynamic}, which has $O(K)$ complexity. In phase 2, we apply automatic differentiation. It is well-known that the complexity of calculating derivatives with automatic differentiation has the same time complexity as of the function evaluation, which can also be verified from the step~5~-~9 in Algorithm~\ref{alg:overall-dynamic}. So the overall time complexity is $O(K)$. As a result, we maintain the same $O(K)$ time complexity as those gradient-based methods (see Appendix A for a reference of these methods) used for solving the basic bilevel model, though we may have a larger const factor, as we need to extraly evaluate $\nabla_{1} g(\omega_{k},\lambda)$ in every BiG-SAM step, however, we believe this is acceptable compared to the merits of our model. In fact, we can avoid calculating $\nabla_{1} g(\omega_{k},\lambda)$ to save computation by setting $\alpha_k$ as 1. We will elaborate more about this in the ablation study.

\section{Theoretical Analysis}
\label{sec:ta}
In this section, we introduce two important properties of our model. In Theorem~\ref{th:convergence}, we study the convergence property of Algorithm~\ref{alg:overall-dynamic}. Next in Theorem~\ref{theo:lb}, we show that the minimum of our model is guaranteed to be less than that of the basic model, this result directly shows the advantage of our model in optimizing the outer objective. The necessary assumptions made in the proof is summarized as follows:

\begin{enumerate}[label=(\roman*)]
    \item $\Lambda \subset R^{m}$ is a compact set;
    \item $\Omega  \subset R^{n}$ is a compact set;
    \item g($\omega$, $\lambda$) is Lipschitz continuous with constant $L_g$, Lipschitz differentiable with constant $L_{\nabla g}$, and g($\sbullet$, $\lambda$) is uniformly strongly convex with parameter $\sigma_g$;
    \item h($\omega$, $\lambda$) is Lipschitz continuous with constant $L_h$, Lipschitz differentiable with constant $L_{\nabla_h}$; h($\sbullet$, $\lambda$) is convex and $\nabla^2 h(\sbullet, \lambda)$ exists and is bounded below with constant $S_{\nabla^2 h}$ when $\nabla h(\sbullet, \lambda) \ne 0$;
\end{enumerate}

Now we discuss the convergence property of Algorithm~\ref{alg:overall-dynamic}. For the simplicity of notation, we use $f(\lambda)$ to denote $g(\omega_{\lambda},\lambda)$ where $\omega_{\lambda}$ denotes the exact solution of the inner prblem. Similarly, we use $f_K(\lambda)$ to denote $g(\hat{\omega}_{\lambda},\lambda)$ where $\hat{\omega}_{\lambda}$ is the output of Algorithm~\ref{alg:inner_dynamic}. Before getting into the details of the main theorem, we firstly show the existence of minimizer of $f(\lambda)$. In fact, due to additional regularization in the constraint of Eq.~(\ref{eq:hier-bilevel}), $f(\lambda)$ may be discontinuous, and therefore dose not attain the minimum. However, Proposition~\ref{prop1} and Lemma~\ref{lemma1} show that the continuity property holds under our assumption. Firstly, Proposition~\ref{prop1} shows a result about the convergence of $\{\Omega^{*}_{\lambda_n}\}_{n\in N}$, more precisely :
\begin{proposition}\label{prop1}
let $(\lambda_n)_{n\in N}$ be a sequence converges to $\lambda$, with the above assumptions (i) - (iv) hold, $\Omega^{*}_{\lambda}$ is the Kuratowski limit \cite{kuratowski2014topology} of $\{\Omega^{*}_{\lambda_n}\}_{n\in N}$, \emph{i.e.}:

\begin{equation}
\begin{split}
    \underset{n\to\infty}{Li} \Omega^{*}_{\lambda_n} &= \{\omega\in\Omega^*_{\lambda}|\underset{n\to\infty}{\limsup}\ d(\omega, \Omega^{*}_{\lambda_n})=0\}=\Omega^*_{\lambda}\\
    \underset{n\to\infty}{Ls} \Omega^{*}_{\lambda_n} &= \{\omega\in\Omega^*_{\lambda}|\underset{n\to\infty}{\liminf}\ d(\omega, \Omega^{*}_{\lambda_n})=0\}=\Omega^*_{\lambda}\\
\end{split}
\end{equation}
where:\\
\begin{equation}
\begin{split}
    \Omega^{*}_{\lambda} = \{\omega, \omega\in \arg\min\ h(\omega,\lambda)\}\ ,\Omega^{*}_{\lambda_n} = \{\omega, \omega\in \arg\min\ h(\omega,\lambda_n)\}
\end{split}
\end{equation}
\end{proposition}

\begin{hproof}
Follow the definition of Kuratowski convergence, we need to prove the Kuratowski limits inferior and superior of the sequence agree for every point $\omega$ in $\Omega^{*}_{\lambda}$, which we show by proving $\underset{n\to\infty}{\lim}\ d(\omega, \Omega^{*}_{\lambda_n})$ exists and equals 0, the detailed proof is shown in Appendix D.
\end{hproof}

With Proposition~\ref{prop1}, we prove the the existence of minimizer of $f(\lambda)$ in the following lemma:

\begin{lemma}\label{lemma1} With the above assumptions (i) - (iv) hold, $f(\lambda)$ is continuous in $\Lambda$ and admits a minimizer.
\end{lemma}

\begin{hproof}
To prove the continuity of $f(\lambda)$, for any $\lambda \in \Lambda$, and a sequence $(\lambda_n)_{n\in N} \to \lambda$, we need to prove $f(\lambda_n)_{n\in N} \to f(\lambda)$, which relies on the convergence result in Proposition~\ref{lemma1}. As for the existence of the minimizer of $f(\lambda)$, since $\Lambda$ is compact and $f(\lambda)$ is continuous, we can easily get $f(\lambda)$ attains a minimum by the extreme value theorem. The detailed proof is shown in the Appendix~D.
\end{hproof}

Now that $f(\lambda)$ is continuous and attains a minimizer, we can prove the convergence property of Algorithm~\ref{alg:overall-dynamic}, where we show by the convergence of $f_{K}(\lambda)$ to $f(\lambda)$:
\begin{theorem}\label{th:convergence}
With the assumptions (i)-(iv) above hold, 
we have as $K \to \infty$ for all $\lambda \in \Lambda$:
\begin{enumerate}[label=(\roman*)]
    \item $\min f_K(\lambda) \to \min f(\lambda)$;
    \item $\arg\min f_K(\lambda) \to \arg\min f(\lambda)$;
\end{enumerate}
\end{theorem}
\begin{hproof}
if we can prove $f_{K}(\lambda) \to f(\lambda)$ uniformly as $K\to\infty$, then the two property is a direct consequence of the stability of minimum and minimizers in optimization  \cite{dontchev2006well, franceschi2018bilevel}. To prove the above uniform convergence, we make use of the Lipschitz continuity of $g(\omega,\lambda)$ and the uniform convergence of $\omega_K \to \omega_{\lambda}$,
A more precise and rigorous proof is shown in Appendix D.
\end{hproof}

Theorem~\ref{th:convergence} states that the output of Algorithm~\ref{alg:overall-dynamic} will converge to the minimizer of $f(\lambda)$ in the limit, which shows the effectiveness of the algorithm. Next we present Theorem~\ref{theo:lb}, which compares the minimum of our model and that of the basic model:
\begin{theorem}\label{theo:lb}
Let $\lambda_{c}$ and $\lambda_{h}$ denote the minimizer of Eq.~(\ref{eq:classical-bilevel}) and Eq.~(\ref{eq:hier-bilevel}) respectively. Then the minimum of Eq.~(\ref{eq:hier-bilevel}) is less than that of Eq.~(\ref{eq:classical-bilevel}), i.e.:
\begin{equation}
    g(\omega_{h,\lambda_h},\lambda_{h}) \le g(\omega_{c,\lambda_c},\lambda_{c})
\end{equation}
where $\omega_{h,\lambda}$ denotes the solution of inner problem for a fixed $\lambda$ in Eq.~(\ref{eq:hier-bilevel}) and $\omega_{c,\lambda}$ denotes the solution of inner problem for a fixed $\lambda$ in Eq.~(\ref{eq:classical-bilevel}).
\end{theorem}
\begin{proof}
Based on the formulation in Eq.~(\ref{eq:inner}), we have:
\begin{equation}
    \omega_{h,\lambda} = \underset{\omega\in\Omega^{*}_{\lambda}}{\arg\min}\ g(\omega,\lambda)
\end{equation}
for each fixed $\lambda$ and since $\omega_{c,\lambda} \in \Omega^{*}_{\lambda}$ so we have:\\
\begin{equation}
\label{eq:lower bound}
    g(\omega_{h,\lambda},\lambda) \le g(\omega_{c,\lambda},\lambda)
\end{equation}
for all $\lambda$, naturally we have:
\begin{equation}
\begin{split}
  g(\omega_{h,\lambda_h},\lambda_{h}) = \underset{\lambda\in\Lambda}{\min}\ g(\omega_{h,\lambda},\lambda) \le \underset{\lambda\in\Lambda}{\min}\ g(\omega_{c,\lambda},\lambda) =g(\omega_{c,\lambda_c},\lambda_{c}) \\
\end{split}
\end{equation}
which proves the desired result.
\end{proof}
Note that we use the claim $\omega_{c,\lambda} \in \Omega^{*}_{\lambda}$ in the above proof. This is based on the fact that algorithms used to solve Eq.~(\ref{eq:classical-bilevel}) usually find a random local minimizer of the inner problem.  In contrast, we choose the "best" minimizer of the inner problem. This is the key that leads to the result stated in Theorem~\ref{theo:lb}. In fact, there is a great gap between the minimums found by the two models in practice, which we will show in the experiments shortly.

\section{Experimental Results}
\label{sec:er}
In this section, we present the experimental results and compare the improved bilevel model with the basic bilevel model. We test over two tasks, \emph{i.e.} Data Hyper-Cleaning and Hyper-Representation Learning. For the basic bilevel model, we apply a gradient-based algorithm similar to our Algorithm~\ref{alg:overall-dynamic} where we replace Algorithm~\ref{alg:inner_dynamic} with a gradient descent optimizer, the detailed description of the algorithm used for the basic bilevel model is shown in Appendix B.2. The hyper-parameters of the experiments for both models are kept the same for fair comparison (hyper-parameter settings for each experiment are shown in Appendix C.1 and C.2). The code is written with Pytorch \cite{paszke2019pytorch} and run on a server with Tesla P40 GPU.

\subsection{Data Hyper-Cleaning}
The data hyper-cleaning task aims to clean a noisy data-set where the labels of some samples are corrupted.
One way to tackle this problem is to learn a "weight" for each sample and treat samples with negative "weight" as corrupted samples. To get these weights, we fit a weighted model on the noisy data-set and tune the weights over a clean validation data-set. This is a bilevel problem as the weighted model needs to be learned before evaluated over the validation data-set. The precise mathematical formulation of this method is shown in Appendix C.1.
We run experiments over five datasets, including MNIST \cite{lecun2010mnist}, Fashion-MNIST \cite{xiao2017fashion}, QMNIST \cite{yadav2019cold}, CIFAR10 \cite{krizhevsky2009learning} and SVHN~\cite{netzer2011reading}. We handcraft the noisy data-set by randomly perturbing the labels of a given percentage of samples (we denote this percentage as $\rho$, \emph{e.g.} $\rho = 0.8$ represents the label of 80\% samples in the training set are corrupted). The details of data-set construction can also be found in Appendix C.1.

\begin{table*}[h]
\vskip 0.15in
\caption{Comparison of F1 score between the improved bilevel model and the basic model}
\centering
\begin{small}
\resizebox{\columnwidth}{!}{
\begin{tabular}{@{}cccccccccccc@{}}
\toprule
\multirow{2}{*}{F1-score} &  
\multicolumn{2}{c}{$\rho = 0.9$} & \phantom{a} & \multicolumn{2}{c}{$\rho = 0.8$} & \phantom{a} & \multicolumn{2}{c}{$\rho = 0.6$}\\
\cmidrule{2-3} \cmidrule{5-6} \cmidrule{8-9}& Improved & Basic && Improved & Basic && Improved & Basic\\
\midrule
MNIST & \textbf{0.618$\pm$0.008} & 0.527$\pm$0.008 && \textbf{0.781$\pm$0.014} & 0.721$\pm$0.011 && \textbf{0.898$\pm$0.004} & 0.877$\pm$0.005\\
Fashion-MNIST & \textbf{0.569$\pm$0.005} & 0.519$\pm$0.012 && \textbf{0.736$\pm$0.008} & 0.697$\pm$0.005 && \textbf{0.851$\pm$0.004} & 0.839$\pm$0.003\\
Q-MNIST & \textbf{0.615$\pm$0.007} & 0.531$\pm$0.010 && \textbf{0.780$\pm$0.005} & 0.721$\pm$0.004 && \textbf{0.900$\pm$0.006} & 0.873$\pm$0.004\\
SVHN & \textbf{0.404$\pm$0.025} & 0.279$\pm$0.009 && \textbf{0.536$\pm$0.010}& 0.443$\pm$0.015 && \textbf{0.646$\pm$0.013}& 0.614$\pm$0.010\\
CIFAR10 & \textbf{0.321$\pm$0.014} & 0.289$\pm$0.007 && \textbf{0.499$\pm$0.008}& 0.456$\pm$0.011&& \textbf{0.649$\pm$0.007}& 0.621$\pm$0.010\\
\bottomrule
\end{tabular}}
\end{small}
\label{tb:result}
\vskip -0.1in
\end{table*}

Now we present some experimental results. We use F1-score
as the metric of the data cleaning algorithm.
From Table \ref{tb:result}, we can see that the F1-scores of our model surpass the basic model over different data-sets and under different hardness levels with a great margin (We report mean and variance over five runs for each experiment). Moreover, the performance gap between our model and the basic model becomes larger as the task becomes harder (larger $\rho$). For harder tasks, it is more important to make use of the clean validation set's information, therefore, this experiment verifies furthermore that our model is much better at exploiting the information of the outer problem, which is an important reason of the superior performance of our model.
What's more, we plot the F1 score curve of MNIST data-set (first row in Table~\ref{tb:result}) in Fig.~\ref{gho-rho} as an example, from which we can see our model converges much faster than the basic model.

\begin{figure}[ht]
\begin{center}
\begin{subfigure}{.31\textwidth}
\includegraphics[width=\linewidth]{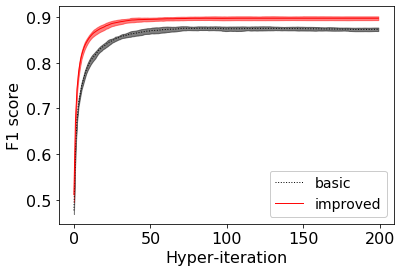}
\caption{$\rho$ = 0.6}
\end{subfigure}
\begin{subfigure}{.31\textwidth}
\includegraphics[width=\linewidth]{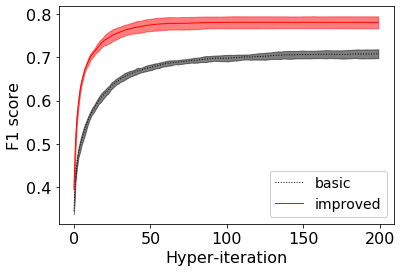}
\caption{$\rho$ = 0.8}
\end{subfigure}
\begin{subfigure}{.31\textwidth}
\includegraphics[width=\linewidth]{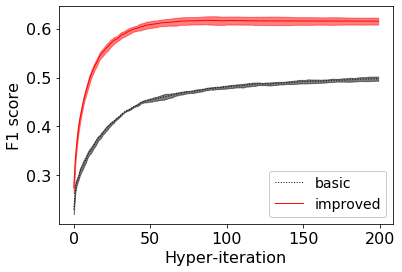}
\caption{$\rho$ = 0.9}
\end{subfigure}
\caption{F1 score curve of MNIST data-set under different noise levels ($\rho$), larger $\rho$ means harder data-set.}
\label{gho-rho}
\end{center}
\vskip -0.2in
\end{figure}


\begin{figure*}[t]
\begin{center}
\begin{subfigure}{.4\textwidth}
\includegraphics[width=\linewidth]{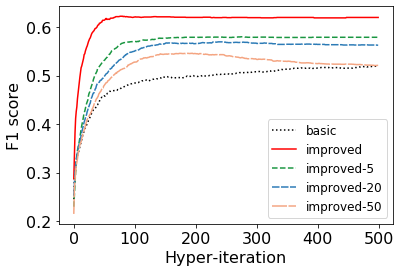}
\caption{MNIST}
\end{subfigure}
\qquad
\begin{subfigure}{.4\textwidth}
\includegraphics[width=\linewidth]{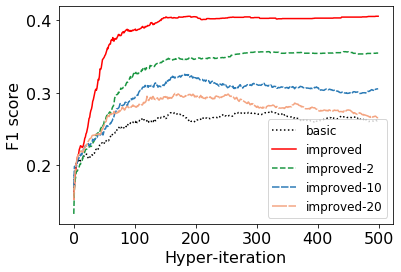}
\caption{SVHN}
\end{subfigure}
\end{center}
\caption{Applying Big-SAM step with different frequencies, "improved-k" means applying BiG-SAM step every k iterations and using gradient descent step otherwise.}
\label{gho-res}
\end{figure*}

\begin{figure*}[t]
\begin{center}
\begin{subfigure}{.4\textwidth}
\includegraphics[width=\linewidth]{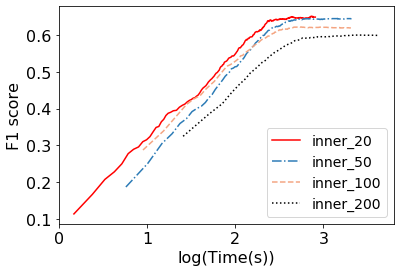}
\caption{MNIST}
\end{subfigure}
\qquad
\begin{subfigure}{.4\textwidth}
\includegraphics[width=\linewidth]{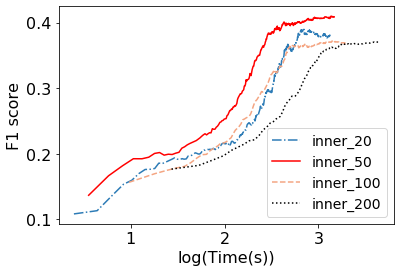}
\caption{SVHN}
\end{subfigure}
\end{center}
\caption{F1 score curve under different number of inner iterations. Note the unit of X-axis is log(Time)}
\label{gho-inner}
\vskip -0.2in
\end{figure*}


\paragraph{Ablation Study.} 
Although our model are mathematically distinct from the basic model, the gradient-based algorithms used to solve these two models share similarities in many aspects. As mentioned in Section~\ref{sec:alg}, if we set $\alpha_k$ to be 1, Algorithm~\ref{alg:inner_dynamic} becomes a gradient descent optimizer and could be used to solve the basic model.
We thus design the following experiments: we replace the BiG-SAM step with ordinary gradient descent step by setting $\alpha_k$ as 1 in Algorithm~\ref{alg:inner_dynamic} and see how the F1 score curve changes. Not surprisingly, the performance of our model degrades when the BiG-SAM step is used less often as shown in Fig.~\ref{gho-res}, however, our model still outperforms the basic model, \emph{e.g.}, when our model performs the BiG-SAM step every 20 steps (only five times if the total steps are 100), our model still outperforms the basic model slightly. This experiment is a great demonstration of the effectiveness of our model, it directly shows how the modified inner problem in our model improves the model performance. Besides, this can also be used as the trade-off between model performance and computation overhead. The more frequent we perform BiG-SAM steps, the better the model performance is, but with more computation overhead, and vice versa.

Finally, we study the relationship between number of inner iterations and model performance.
As shown in Fig.~\ref{gho-inner}, when we train 200 iterations (note that the inner problem is not overfitting in this case), the F1 curve increases much slower and even performs slightly worse compared to that of using less inner iterations. This could be explained by the long back propagation path and the sparse use of outer problem's information, moreover, this also demonstrates that optimizing inner problem solely is not optimal for the whole bilevel problem.

\subsection{Hyper Representation Learning}
We now present the experimental results of the hyper representation learning task. 
Representation learning aims to learn a universal representational mapping that is useful for a bunch of tasks. One way to tackle this problem is to phrase it as bilevel programming. \emph{i.e.} learning the representational mapping is the outer problem, while fitting a simple (linear) model on top of this mapping in a specific task is the inner problem. The precise mathematical formulation can be found in Appendix C.2. For this task, we benchmark over two data-sets, \emph{i.e.} Omniglot \cite{lake2015human} and MiniImageNet \cite{vinyals2016matching}. The detailed construction of training and validation set is also shown in Appendix C.2.

\begin{figure}[!t]
\centering
\begin{subfigure}{.4\textwidth}
\includegraphics[width=\columnwidth]{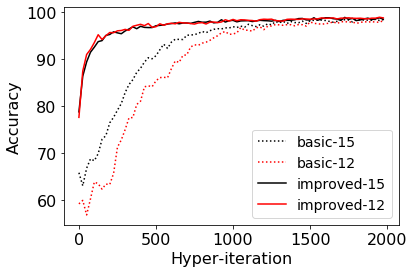}
\caption{5-way 1-shot}
\end{subfigure}
\qquad
\begin{subfigure}{.4\textwidth}
\includegraphics[width=\columnwidth]{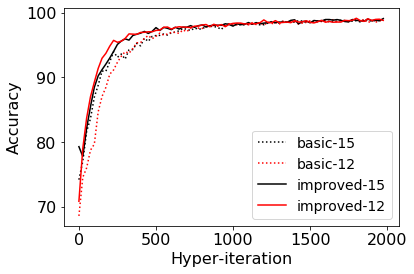}
\caption{5-way 5-shot}
\end{subfigure}\\
\begin{subfigure}{.4\textwidth}
\includegraphics[width=\columnwidth]{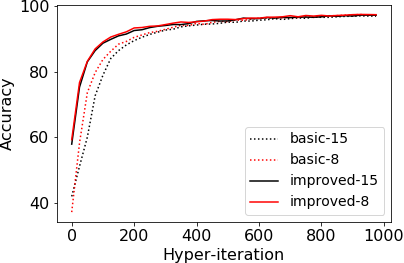}
\caption{20-way 1-shot}
\end{subfigure}
\qquad
\begin{subfigure}{.4\textwidth}
\includegraphics[width=\columnwidth]{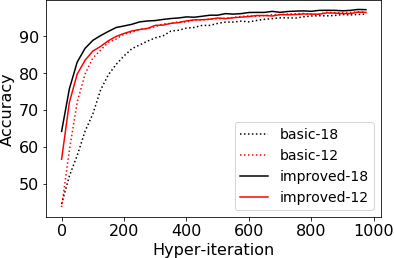}
\caption{20-way 5-shot}
\end{subfigure}
\caption{Accuracy curve for the Omniglot Data-set. "basic-k" means basic bielvel model with k inner iterations, plots of the improved model follow similar rules. }
\label{omni-res}
\end{figure}

\begin{figure}[!t]
\centering
\begin{subfigure}{.4\textwidth}
\includegraphics[width=\columnwidth]{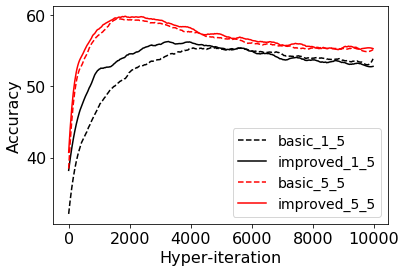}
\caption{5 way}
\end{subfigure}
\qquad
\begin{subfigure}{.4\textwidth}
\includegraphics[width=\columnwidth]{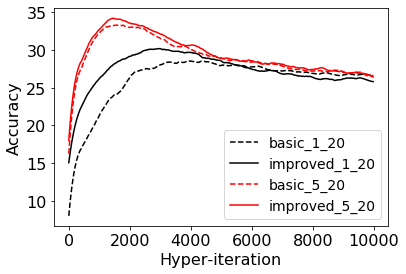}
\caption{20 way}
\end{subfigure}
\caption{Accuracy curve for the MiniImageNet Data-set. "classical-k-n" means basic bilevel model with n-way k-shot, plots of improved model follow similar rules.}
\label{mini-res}
\vskip -0.2in
\end{figure}
In Fig.~\ref{omni-res} and Fig.~\ref{mini-res}, we present the experimental results for Omniglot and MiniImageNet. We test under various settings such as number of classes (way), number of examples per class (shot) and number of inner iterations. As shown in figures, our model achieves better accuracy and converges faster than the basic bilevel model under all settings. For example, for the 20-way 1-shot case of the MiniImageNet, our model reaches the accuracy of 30\%, but the algorithm based on the basic bilevel model gets only about 26\%. What's more, similar to the cleaning task, our model has the greatest performance gain under the hardest setting, \emph{i.e.} the 1-shot case for both data-sets. This phenomenon demonstrates oncemore the effectiveness of our model.

\section{Conclusion}
\label{sec:c}
In this paper, we propose an improved bilevel model in place of the current bilevel formulation, and show both theoretically and empirically the superior performance of our model. We show that our model converges to a lower optimal value and makes more efficient use of the outer problem's information. Moreover, we empirically demonstrate the superior performance of our model over the data hyper-cleaning task and the hyper representation learning task. As many machine learning problems have a bilevel structure, our new model is in great potential to be applied to solve many related applications.
\section*{Broader Impact}

Bilevel programming could be applied to a wide range of machine learning applications, such as hyper-parameter optimization, few-shot learning and generative adversarial learning. Our research focuses on improving the current bilevel formulation.
Generally speaking, our new formulation can lead to a better and faster convergence compared to the current bilevel formulation. As a result, the applications that apply bilevel programming will benefit from our method. Note that our method is not based on any bias in the data-set and has not direct ethical consequences, however, our method is quite general and can be applied to many potential applications, so we encourage researchers to apply our new formulation to many applications and investigate the various ethical effects of our model.


\bibliography{neurips_2020}
\bibliographystyle{abbrvnat}



\appendix

\section{More Related Work}
Eq.~(\ref{eq:opt-bilevel}) has been extensively investigated since it is firstly proposed by \citet{willoughby1979solutions}. More specifically, \citet{willoughby1979solutions} proposes to solve the following problem as a proxy of Eq.~(\ref{eq:opt-bilevel}):
\begin{equation}
\label{eq:reg-bilevel}
\begin{split}
    \underset{\omega\in \Omega}{\min}\ h(\omega) + \lambda g(\omega)
\end{split}
\end{equation}
where $\lambda$ is a regularization hyper-parameter; \citet{ferris1991finite, solodov2007explicit} focus on the regularity conditions of the involved functions and investigate the scenario when $h(\omega)$ is the sum of a smooth function and an indicator function; \citet{yamada2011minimizing} focuses on finding a reasonable regularization parameter $\lambda$ and proposes sufficient conditions of $\lambda$ \emph{s.t.} the solution of Eq.~(\ref{eq:reg-bilevel}) converges to that of Eq.~(\ref{eq:opt-bilevel});
\citet{sabach2017first} proposes an algorithm to solve Eq.~(\ref{eq:opt-bilevel}) directly named BiG-SAM with guaranteed convergence rate of $O(\frac{1}{k})$. 
More specifically, BiG-SAM is a first order method making use of the Sequential Averaging Method \cite{xu2004viscosity} which is originally proposed for fixed point problems. Its procedure is shown in Algorithm~\ref{alg:big-sam}. 

\begin{algorithm}[H]
\caption{Bilevel Gradient Sequential Averaging Method (BiG-SAM)}
\label{alg:big-sam}
\begin{algorithmic}[1]
\STATE {\bfseries Input:} Initial inner variable $\boldsymbol{\omega_{0}}$;\ number of inner iterations \textbf{K};\ hyper-parameters \textbf{t},\ \textbf{s}\ and $\boldsymbol{\alpha}$;
\STATE {\bfseries Output:} Optimal solution $\boldsymbol{\omega}$;
\FOR{k = 0 $\to$ K}
\STATE $\theta_{k+1} = \omega_{k} - t\nabla_{1} h(\omega_{k})$
\STATE $\phi_{k+1} = \omega_{k} - s\nabla_{1} g(\omega_{k})$
\STATE $\omega_{k+1} = \alpha_{k+1}\theta_{k+1} + (1 - \alpha_{k+1}) \phi_{k+1}$
\ENDFOR
\STATE {\bfseries Return} $\omega = \omega_{K}$
\end{algorithmic}
\end{algorithm}


In machine learning, we study Eq.~(\ref{eq:classical-bilevel}) as introduced in Section~\ref{sec:rw}, for which the
best-known methods are gradient-based methods \cite{franceschi2017forward, shaban2018truncated, franceschi2018bilevel}.
Here we present an example of the reverse-mode fully back-propagated method in Algorithm \ref{alg:reverse-mode}. 
\begin{algorithm}[H]
  \caption{Reverse mode fully back-propagated method}
  \label{alg:reverse-mode}
\begin{algorithmic}[1]
  \STATE {\bfseries Input:} Initial outer variable $\boldsymbol{\lambda_0}$;\ number of iterations \textbf{K};\ hyper-parameters $\boldsymbol{\eta_o}$,\ $\boldsymbol{\eta_i}$
  \STATE {\bfseries Output:} Optimal outer variable $\boldsymbol{\lambda_{T}}$;
  \FOR{m = 1 $\to$ T\ -\ 1}
  \FOR{k = 0 $\to$ K}
  \STATE Update $\omega_k$ using gradient descent with learning rate $\eta_i$ (whose dynamics is denoted as $\Phi(\omega_{k},\lambda_m)$)
  \ENDFOR  
  \STATE Let $\hat{\omega}_{\lambda_m} = \omega_K$,\ $\alpha = \nabla_{1} g(\hat{\omega}_{\lambda_m}, \lambda_m)$, $G = \nabla_{2} g(\hat{\omega}_{\lambda_m}, \lambda_m)$
  \FOR{k = K - 1 $\to$ 1}
  \STATE $G = G\  +\  \alpha \nabla_2 \Phi(\omega_{k+1},\lambda_m)$,\ $\alpha\ =\  \alpha \nabla_1 \Phi(\omega_{k+1},\lambda_m)$
  \ENDFOR
  \STATE $\lambda_{m+1} = \lambda_{m} - \eta G$ 
  \ENDFOR
  \STATE {\bfseries Return} $\lambda_T$
\end{algorithmic}
\end{algorithm}


\section{Details of Algorithms}
In this section, we include more details about the proposed algorithms in Section~\ref{sec:ibm}, moreover, we also present the algorithm we use to solve the basic bilevel model in Section~\ref{sec:er}, for the purpose of illustrating the differences between our model and the basic model from an algorithmic perspective.

\subsection{Algorithm of the Improved Bilevel Model}
As mentioned in Section~\ref{sec:alg}, there are two phases within each hyper-iteration. Firstly, we solve the inner problem using Algorithm~\ref{alg:inner_dynamic}, next we apply automatic differentiation to compute the hyper-gradient. For concision, we combine Algorithm~\ref{alg:inner_dynamic} and Algorithm~\ref{alg:overall-dynamic} together here in Algorithm~\ref{alg:combined-dynamic}.

Now we derive the procedure to compute hyper-gradient based on automatic-differentiation. Firstly, we give the precise form of dynamics $\Phi(\omega,\lambda)$
and its partial derivatives, more specifically, the BiG-SAM step update is:
\begin{equation}
\begin{split}
    \omega_{k+1} = \omega_{k} - t\alpha_{k+1}\nabla_{1} h(\omega_{k},\lambda) - s(1 - \alpha_{k+1})\nabla_{1} g(\omega_{k},\lambda)
\end{split}
\end{equation}
Naturally, its derivatives are defined as follows:
\begin{equation}
\begin{split}
    \nabla_1 \Phi(\omega_{k},\lambda) &= I -  t\alpha_{k+1}\nabla_{11} h(\omega_{k},\lambda) - s(1 - \alpha_{k+1})\nabla_{11} g(\omega_{k},\lambda)\\
    \nabla_2 \Phi(\omega_{k},\lambda) &=  -t\alpha_{k+1}\nabla_{12} h(\omega_{k},\lambda) - s(1 - \alpha_{k+1})\nabla_{12} g(\omega_{k},\lambda)
\end{split}
\end{equation}
where $I$ denotes the unit matrix. Finally, the hyper-gradient can be calculated by combining the above results and the chain-rule as:
\begin{equation}
\label{eq:hyper}
\begin{split}
    \nabla g(\hat{\omega}_{\lambda},\lambda)& = \nabla_{2} g(\hat{\omega}_{\lambda}, \lambda)+\\
    &\nabla_{1} g(\hat{\omega}_{\lambda}, \lambda)\times\left(\left(\sum_{k=1}^{K-2}\left(\prod_{j=k+1}^{K-1} \nabla_1 \Phi(\omega_{j},\lambda)\right)\nabla_2 \Phi(\omega_{k},\lambda)\right)+ \nabla_2 \Phi(\omega_{K-1},\lambda)\right)
\end{split}
\end{equation}
which corresponds to Line 10 - 13 in Algorithm~\ref{alg:combined-dynamic} or Line 5 - 9 in Algorithm~\ref{alg:overall-dynamic}.
\begin{algorithm}[H]
  \caption{Algorithm of the improved bilevel model}
  \label{alg:combined-dynamic}
\begin{algorithmic}[1]
  \STATE {\bfseries Input:} Initial outer variable $\boldsymbol{\lambda_0}$,\ number of iterations \textbf{K},\ hyper-parameter $\boldsymbol{\eta}$, \textbf{t},\ \textbf{s}\ and $\boldsymbol{\alpha}$;
  \STATE {\bfseries Output:} Optimal outer variable $\boldsymbol{\lambda_{T}}$;
  \FOR{m = 1 to T\ -\ 1}
  \STATE {\bfseries Initialize:} Inner variable $\boldsymbol{\omega_{0}}$;
  \FOR{k = 0 $\to$ K} \STATE (BiG-SAM step)
    \STATE $\theta_{k+1} = \omega_{k} - t\nabla_{1} h(\omega_{k},\lambda_m)$,\ $\phi_{k+1} = \omega_{k} - s\nabla_{1} g(\omega_{k},\lambda_m)$
    \STATE $\omega_{k+1} = \alpha_{k+1}\theta_{k+1} + (1 - \alpha_{k+1}) \phi_{k+1}$
  \ENDFOR
  \STATE Let $\hat{\omega}_{\lambda_m} = \omega_{K}$,\ $\alpha = \nabla_{1} g(\hat{\omega}_{\lambda_m}, \lambda_m)$, $G = \nabla_{2} g(\hat{\omega}_{\lambda_m}, \lambda_m)$
  \FOR{k = K - 1 $\to$ 1}
  \STATE $G = G + \alpha \nabla_2 \Phi(\omega_{k+1},\lambda_m)$,\ $\alpha = \alpha \nabla_1 \Phi(\omega_{k+1},\lambda_m)$
  \ENDFOR
  \STATE $\lambda_{m+1} = \lambda_{m} - \eta G$ 
  \ENDFOR
  \STATE {\bfseries Return} $\lambda_T$
\end{algorithmic}
\end{algorithm}

\subsection{Algorithm of the Basic Bilevel Model}
Now we introduce the algorithm we use to solve the basic bilevel model. Since the the inner problem of a basic bilevel model is:
\begin{equation}
\label{eq:basic-inner}
\begin{split}
\underset{\omega\in\Omega}{\arg\min}\ h(\omega,\lambda)
\end{split}
\end{equation}
We solve Eq.~(\ref{eq:basic-inner}) with a gradient descent optimizer, while for the calculation of hyper-gradient, we also apply the automatic differentiation. We summarize the algorithm in Algorithm~\ref{alg:basic}. The dynamics $\Phi(\omega, \lambda)$ of the gradient descent used for solving the inner problem Eq.~(\ref{eq:basic-inner}) is:
\begin{equation}
\label{eq:dy-inner}
\begin{split}
    \omega_{k+1} = \omega_{k} - t\nabla_{1} h(\omega_{k},\lambda)
\end{split}
\end{equation}
and its derivatives are defined as follows:
\begin{equation}
\label{eq:grad-inner}
\begin{split}
    \nabla_1 \Phi(\omega_{k},\lambda) = I -  t\nabla_{11} h(\omega_{k},\lambda),\ 
    \nabla_2 \Phi(\omega_{k},\lambda) =  -t\nabla_{12} h(\omega_{k},\lambda)
\end{split}
\end{equation}
We can calculate the hyper-gradient of the basic bilevel model easily by substituting Eq.~(\ref{eq:dy-inner}) and Eq.~(\ref{eq:grad-inner}) into Eq.~(\ref{eq:hyper}).
\begin{algorithm}[H]
  \caption{Algorithm of the basic bilevel model}
  \label{alg:basic}
\begin{algorithmic}[1]
  \STATE {\bfseries Input:} Initial outer variable $\boldsymbol{\lambda_0}$,\ number of iterations \textbf{K},\ hyper-parameter $\boldsymbol{\eta}$, \textbf{t};
  \STATE {\bfseries Output:} Optimal outer variable $\boldsymbol{\lambda_{T}}$;
  \FOR{m = 1 to T\ -\ 1}
  \STATE {\bfseries Initialize:} Inner variable $\boldsymbol{\omega_{0}}$;
  \FOR{k = 0 $\to$ K}
    \STATE $\omega_{k+1} = \omega_{k} - t\nabla_{1} h(\omega_{k},\lambda_m)$
  \ENDFOR
  \STATE Let $\hat{\omega}_{\lambda_m} = \omega_{K}$,\ $\alpha = \nabla_{1} g(\hat{\omega}_{\lambda_m}, \lambda_m)$, $G = \nabla_{2} g(\hat{\omega}_{\lambda_m}, \lambda_m)$
  \FOR{k = K - 1 $\to$ 1}
  \STATE $G = G + \alpha \nabla_2 \Phi(\omega_{k+1},\lambda_m)$,\ $\alpha = \alpha \nabla_1 \Phi(\omega_{k+1},\lambda_m)$
  \ENDFOR
  \STATE $\lambda_{m+1} = \lambda_{m} - \eta G$ 
  \ENDFOR
  \STATE {\bfseries Return} $\lambda_T$
\end{algorithmic}
\end{algorithm}
\section{Experimental Details}

\subsection{Data Hyper-Cleaning} 

In the data hyper-cleaning task, the inner and outer problem are defined as follows:
\begin{equation}
    g(\omega,\lambda) = \sum_{i=1}^{N_{val}} l(y_{i,val}, x_{i,val}\times\omega)\,, \;\;
    h(\omega,\lambda) = \sum_{i=1}^{N_{tr}}\sigma(\lambda_i) \times l(y_{i,tr}, x_{i,tr}\times\omega)\,,
\end{equation}
where $\{(x_{i,tr},y_{i,tr})_{i\in[N_{tr}]}\}$ denotes the training set and $\{(x_{i,val},y_{i,val})_{i\in[N_{val}]}\}$ denotes the validation set. The outer variable $\{\lambda_i\}_{i\in[N_{tr}]}$ is the un-normalized weights for training samples and $\sigma(\sbullet)$ represents the sigmoid function, while the inner variable $\omega$ is weights of the machine learning model and $l$ is the loss function, \emph{e.g.} cross entropy loss. Note that the number of hyper-parameters in this task is of order $O(N_{tr})$, which can be very large depending on the scale of the data-set to be cleaned.

\noindent\textbf{Data construction.} We randomly select 5000 samples to construct the training set and corrupt the label of $\rho$ percentage samples, then we randomly select another 5000 samples to construct the validation set. 

\noindent\textbf{Hyper-parameters.} The model we use for this task is based on the ResNet-18 architecture. For the experiments in Table~\ref{tb:result} and Fig.~\ref{gho-rho}, the inner iterations are set 100 and the outer problem is run until converge. In the improved bilevel model, we set both $t$ and $s$ as 0.001, $\alpha$ as $k^{-0.25}$, and $\eta$ as 1; while in the basic model, we set $t$ as 0.001 and $\eta$ as 1. The learning rates of both models are kept the same on propose for fair comparison. In Fig.~\ref{gho-res}, the hyper-parameters are kept the same other than setting $\alpha_k$ as 1 occasionally. In Fig.~\ref{gho-inner}, we vary the number of inner iterations but keep the other hyper-parameters the same.

\subsection{Hyper Representation Learning} 

In the hyper representation learning, the inner and outer problem are defined as follows:
\begin{equation}
    g(\omega,\lambda) = \sum_{i=1}^{N_{t}} l(m(D_{i,val};\lambda); \omega_i)\,,\;\;
    h(\omega,\lambda) = \sum_{i=1}^{N_{t}} l(m(D_{i,tr};\lambda); \omega_i)\,.
\end{equation}
where $\{D_{i,tr}, D_{t,val}\}_{i\in[N_{t}]}$ denotes the tasks used for training. The outer variable $\lambda$ is the parameters of the representation mapping $m$, while the inner variable $\omega_i$ is the weight of linear model for each task. Note that $\lambda$ is shared between tasks and $\omega_i$ is learned independently for each task.


\noindent\textbf{Data construction.} We resize the images (28 by 28 for Omniglot and 64 by 64 for MiniImageNet) and also create more classes by rotating the images (90$^{\circ}$, 180$^{\circ}$ and 270$^{\circ}$). For each training task $D_i$, we include 5 or 20 classes with 1 or 5 training samples and 10 validation samples for each class, which correspond to $D_{i,tr}$ and $D_{i,val}$ respectively. The validation tasks are constructed similarly for final evaluation. 

\noindent\textbf{Hyper-parameters.} For the mapping m, we apply four convolutional layer with 64 filters each for Omniglot related experiments and ResNet-18 for MiniImageNet related experiments, while $l$ is simply a linear mapping. In Fig.~\ref{omni-res}, for the improved bilevel model, we set both $t$ and $s$ as 0.01, $\alpha$ as $k^{-0.25}$, and $\eta$ as 0.001; while for the basic model, we set $t$ as 0.01 and $\eta$ as 0.001. In Fig.~\ref{mini-res}, for the improved bilevel model, we set both $t$ and $s$ as 0.001, $\alpha$ as $k^{-0.25}$, and $\eta$ as 0.0001; while for the basic model, we set $t$ as 0.001 and $\eta$ as 0.0001. The learning rates of both models are kept the same on propose for fair comparison.

\section{Proof of Theorems}
\noindent\textbf{Proposition~\ref{prop1}.} {\it let $(\lambda_n)_{n\in N}$ be a sequence converges to $\lambda$, with the above assumptions (i) - (iv) hold, $\Omega^{*}_{\lambda}$ is the Kuratowski limit \cite{kuratowski2014topology} of $\{\Omega^{*}_{\lambda_n}\}_{n\in N}$, \emph{i.e.}:

\begin{equation}
\begin{split}
    \underset{n\to\infty}{Li} \Omega^{*}_{\lambda_n} &= \{\omega\in\Omega^*_{\lambda}|\underset{n\to\infty}{\limsup}\ d(\omega, \Omega^{*}_{\lambda_n})=0\}=\Omega^*_{\lambda}\\
    \underset{n\to\infty}{Ls} \Omega^{*}_{\lambda_n} &= \{\omega\in\Omega^*_{\lambda}|\underset{n\to\infty}{\liminf}\ d(\omega, \Omega^{*}_{\lambda_n})=0\}=\Omega^*_{\lambda}\\
\end{split}
\end{equation}
where:\\
\begin{equation}
\begin{split}
    \Omega^{*}_{\lambda} = \{\omega, \omega\in \arg\min\ h(\omega,\lambda)\}\ ,\Omega^{*}_{\lambda_n} = \{\omega, \omega\in \arg\min\ h(\omega,\lambda_n)\}
\end{split}
\end{equation}}

\begin{proof}
Since $(\lambda_n)_{n\in N} \to \lambda$, then for any $\epsilon > 0$, there exists $N_{\epsilon}$ s.t. for any $n > N_{\epsilon}$:\\
\begin{equation}
    ||\lambda_n - \lambda|| < \frac{S_{\nabla^2 h}\epsilon}{L_{\nabla h}}
\end{equation}
Then for any point $\omega \in \Omega^{*}_{\lambda}$, since $h$ is Lipschitz differentiable, we have:\\
\begin{equation}
\label{eq:1}
\begin{split}
    &||\nabla_1 h(\omega, \lambda) - \nabla_1 h(\omega, \lambda_n)||
    \le ||\nabla h(\omega, \lambda) - \nabla h(\omega, \lambda_n)||
    \le L_{\nabla h} \times ||(\omega,\lambda) - (\omega,\lambda_n)||\\
    =& L_{\nabla h} \times ||\lambda - \lambda_n||
    < L_{\nabla h} \times \frac{S_{\nabla^2 h}\epsilon}{L_{\nabla h}}
    = S_{\nabla^2 h}\epsilon
\end{split}
\end{equation}
since $\omega \in \Omega^{*}_{\lambda}$, we have:\\
\begin{equation}
\label{eq:2}
    \nabla_1 h(\omega, \lambda) = 0
\end{equation}
combining (\ref{eq:1}) and (\ref{eq:2}) we got:\\
\begin{equation}
\label{eq:res1}
    ||\nabla_1 h(\omega, \lambda_n)|| < S_{\nabla^2 h} \epsilon
\end{equation}
Next denote:\\
\begin{equation}
    d(\omega, \Omega^{*}_{\lambda_n}) = \inf\{d(\omega,a), a\in \Omega^{*}_{\lambda_n}\}\\
\end{equation}
Now we want to prove $d(\omega, \Omega^{*}_{\lambda_n})<\epsilon$ (we can assume $d(\omega, \Omega^{*}_{\lambda_n}) > 0$ otherwise, we have $d(\omega, \Omega^{*}_{\lambda_n}) = 0 < \epsilon$ as needed). \\
Suppose that the above infimum is acquired at $\omega^{*}$, since that $\Omega$ is compact and $h(\omega,\lambda)$ is continuous, so $\Omega^{*}_{\lambda}$ is compact, we have $\omega^{*}\in \Omega^{*}_{\lambda_n}$, \emph{i.e.}:\\
\begin{equation}
    \nabla_1 h(\omega^{*}, \lambda_n) = 0
\end{equation}
What's more, by the mean value theorem, there exists $\omega_0 \notin \Omega^{*}_{\lambda_n}$:\\
\begin{equation}
    \begin{split}
        &\nabla_1 h(\omega, \lambda_n)
        =\nabla_1 h(\omega, \lambda_n) - \nabla_1 h(\omega^{*}, \lambda_n)
        = \nabla_{11} h(\omega_0, \lambda_n) \times (\omega - \omega^{*})
    \end{split}
\end{equation}
since $\omega_0 \notin \Omega^{*}_{\lambda_n}$, we have $\nabla_1 h(\omega_0 ,\lambda) \ne 0$, with assumption iv, $\nabla_{11} h(\omega_0, \lambda_n)$ is invertible, so we have:\\
\begin{equation}
    \omega - \omega_{*} = \nabla_{11} h(\omega_0, \lambda_n)^{-1}\nabla_1 h(\omega, \lambda_n)
\end{equation}
then we have:\\
\begin{equation}
    \begin{split}
        &d(\omega, \Omega^{*}_{\lambda_n})
        =||\omega - \omega^{*}||
        =||\nabla_{11} h(\omega_0, \lambda_n)^{-1}\nabla_1 h(\omega, \lambda_n)||\\
        \le&||\nabla_{11} h(\omega_0, \lambda_n)^{-1}||\times||\nabla_1 h(\omega, \lambda_n)||
        <\frac{1}{S_{\nabla^2 h}} \times S_{\nabla^2 h}\epsilon
        =\epsilon
    \end{split}
\end{equation}
The last inequality is a result of Assumption (iv) and (\ref{eq:res1}), now we get the conclusion that:\\
\begin{equation}
    \underset{n\to\infty}{\lim} d(\omega, \Omega^{*}_{\lambda_n}) = 0
\end{equation}
therefore, we have the Kuritowski limit inferior:\\
\begin{equation}
\begin{split}
    \underset{n\to\infty}{Li} K_n &= \{\omega\in\Omega^*_{\lambda}|\underset{n\to\infty}{\limsup}\ d(\omega, \Omega^{*}_{\lambda_n})=0\}
    =\{\omega\in\Omega^*_{\lambda}|\underset{n\to\infty}{\lim}\ d(\omega, \Omega^{*}_{\lambda_n})=0\}
    =\Omega^*_{\lambda}
\end{split}
\end{equation}
Similarly, we can get $\underset{n\to\infty}{Ls} K_n = \Omega^*_{\lambda}$, therefore we have:\\
\begin{equation}
    \Omega^{*}_{\lambda_n} \to \Omega^{*}_{\lambda}
\end{equation}
\end{proof}
\noindent\textbf{Lemma~\ref{lemma1}.} {\it With the above assumptions (i) - (iv) hold, $f(\lambda)$ is continuous in $\Lambda$ and admits a minimizer.}
\begin{proof}
For any $\lambda \in \Lambda$, and a sequence $(\lambda_n)_{n\in N} \to \lambda$, since $\Omega$ is compact, there exists a subsequence $(k_n)_{n\in N}$ s.t. $\underset{n\to\infty}{\lim} \omega_{\lambda_{k_n}} \to \hat{\omega}$, now if we can prove $\hat{\omega} = \omega_{\lambda}$, then by the continuity of $g(\omega, \lambda)$, we have:\\
\begin{equation}
\label{eq:res4}
\begin{split}
    \underset{n\to\infty}{\lim} f(\lambda_{k_n}) &= \underset{n\to\infty}{\lim} g(\omega_{\lambda_{k_n}}, \lambda_{k_n})
    = g(\omega_{\lambda}, \lambda)
    =f(\lambda)\\
\end{split}
\end{equation}
Now we prove: $\hat{\omega} = \omega_{\lambda}$, for any $w\in\Omega$, firstly we have:\\
\begin{equation}
\label{eq:res2}
    \begin{split}
        h(\hat{\omega}, \lambda) = \underset{n\to\infty}{\lim} h(\omega_{\lambda_{k_n}}, \lambda_{k_n})
        \le \underset{n\to\infty}{\lim} h(\omega, \lambda_{k_n})
        = h(\omega, \lambda)\\
    \end{split}
\end{equation}
The first equality and last equality is by the continuity of $h(\omega, \lambda)$, and the inequality is because of $\omega_{\lambda_{k_n}}\in \Omega^*_{\lambda_{k_n}}$, now we have proved $\hat{\omega}\in \Omega^*_{\lambda}$. Furthermore, since:\\
\begin{equation}
\label{eq:3}
    \omega_{\lambda_{k_n}} = \underset{\omega\in \Omega^{*}_{\lambda_{k_n}}}{\arg\min}\ g(\omega,\lambda_{k_n})\\
\end{equation}
and denote $I_{\Omega^{*}_{\lambda_{k_n}}}$ as the indicator function where $I_{\Omega^{*}_{\lambda_{k_n}}} = 1$ if $\omega \in \Omega^{*}_{\lambda_{k_n}}$ and 0 otherwise, then for any $w\in\Omega$, we have:\\
\begin{equation}
\label{eq:res3}
    \begin{split}
        g(\hat{\omega}, \lambda) &= \underset{n\to\infty}{\lim} g(\omega_{\lambda_{k_n}}, \lambda_{k_n})
        \le \underset{n\to\infty}{\lim} g(\omega, \lambda_{k_n}) + I_{\Omega^{*}_{\lambda_{k_n}}}(\omega)\\
        &\le \underset{n\to\infty}{\lim} g(\omega, \lambda_{k_n}) + I_{\underset{n\to\infty}{\lim} \Omega^{*}_{\lambda_{k_n}}}(\omega)
        = g(\omega, \lambda) + I_{\Omega^{*}_{\lambda}}(\omega)
    \end{split}
\end{equation}
The first inequality is because of (\ref{eq:3}), the second inequality is by the definition of indicator function and the last equality is because of the convergence result in Proposition~\ref{prop1}. with (\ref{eq:res2}) and (\ref{eq:res3}) we have:\\
\begin{equation}
    \hat{\omega} = \underset{\omega\in \Omega^{*}_{\lambda}}{\arg\min}\ g(\omega,\lambda) = \omega_{\lambda}\\
\end{equation}
then by (\ref{eq:res4}) we prove that $f(\lambda)$ is continuous, and since $\Lambda$ is compact and $f(\lambda)$ is continuous, we can easily get $f(\lambda)$ attains its minimum by the extreme value theorem, which proves the desired result.
\end{proof}
\newpage
\noindent\textbf{Theorem~\ref{th:convergence}.} {\it With the assumptions (i) - (iv) above hold, we have as $K \to \infty$:
\begin{enumerate}
    \item $\min f_K(\lambda) \to \min f(\lambda)$;
    \item $\arg\min f_K(\lambda) \to \arg\min f(\lambda)$;
\end{enumerate}}
To prove Theorem~\ref{th:convergence}. we need to cite 2 results (list below as Lemma~\ref{lemma2} and Lemma~\ref{lemma3}).

\begin{lemma}\label{lemma2} (Adapted from proposition 5. in \citet{sabach2017first} and and Theorem 3.2 from \citet{xu2004viscosity})\\
Let $\{\theta_k\}$, $\{\phi_k\}$ and $\{\omega_k\}$ be sequences generated by BiG-SAM:
\begin{enumerate}[label=(\roman*)]
    \item the sequence $\{\omega_k\}_{k\in N}$ converges to $\hat{\omega}$ which satisfies:
            \begin{equation}
                \langle \nabla_1 g(\hat{\omega}, \lambda) , \omega - \hat{\omega}\rangle \ge 0\ ,\forall \omega \in \Omega^{*}_{\lambda}
            \end{equation}
            thus $\hat{\omega} = \omega_{\lambda}$
    \item and the sequence $\{\omega_k\}_{k\in N}$ also satisfies the following inequality, which leads to the above result  :
    \begin{equation}
        ||\omega_{k+1} - \omega_{\lambda}|| \le (1 - \bar{\alpha})||\omega_{k} - \omega_{\lambda}|| +\bar{\alpha}\bar{\beta}
    \end{equation}
    Where $\bar{\alpha} = \bar{\alpha}(L_{\nabla g}, \sigma_g)$ and $\bar{\beta} = \bar{\beta}(\omega_k; L_{\nabla g}, \sigma_g)$ for sufficiently large k;
\end{enumerate}
\end{lemma}

\begin{lemma}\label{lemma3} (Theorem A.1 in \citet{franceschi2018bilevel}) Let $\phi_K$ and $\phi$ be lower semi-continuous functions defined on a compact set $\Lambda$. Suppose that $\phi_K$ converges uniformly to $\phi$ on $\Lambda$ as $K\to\infty$. Then:
\begin{enumerate}[label=(\roman*)]
    \item $\inf \phi_K \to \inf \phi$
    \item $\arg\min \phi_K \to \arg\min \phi$
\end{enumerate}
\end{lemma}
\begin{proof}
With the two lemmas above, we come back to the proof of Theorem~\ref{th:convergence}. From Lemma 2.ii, we know that the convergence of the sequence $\{\omega_k\}_{k\in N}$ depends on $L_{\nabla g}$, $\sigma_g$, which is independent of $\lambda$, so we have that $\{\omega_k\}_{k\in N}$ converges to $\omega_\lambda$ uniformly; then since g($\omega$ ,$\lambda$) is Lipschitz continuous, we have:\\
\begin{equation}
    ||f_{K}(\lambda) - f(\lambda)|| = ||g(\omega_{K}, \lambda) - g(\omega_{\lambda}, \lambda)|| \le L_{g} ||\omega_K - \omega_{\lambda}||
\end{equation}
So we have $f_{K}(\lambda) \to f(\lambda)$ uniformly, what's more, $f_K(\lambda)$ is continuous based on its definition and $f(\lambda)$ is continuous and attains a minimizer because of Lemma~\ref{lemma1}, so we satisfy the conditions in Lemma~\ref{lemma3}, then we can prove the desired result in Theorem~\ref{th:convergence} using Lemma~\ref{lemma3}.\\
\end{proof}

\end{document}